\documentclass{article}

\usepackage[sort&compress,numbers]{natbib}

\usepackage{floatrow}

\usepackage{amsthm}
\usepackage{wrapfig}

\newtheorem{prop}{Proposition}

\usepackage[preprint]{neurips_2025}

\usepackage[utf8]{inputenc} 
\usepackage[T1]{fontenc}    
\usepackage{hyperref}       
\usepackage{url}            
\usepackage{booktabs}       
\usepackage{amsfonts}       
\usepackage{nicefrac}       
\usepackage{microtype}      
\usepackage{xcolor}         
\usepackage{graphicx}
\usepackage{caption}
\usepackage{subcaption}
\usepackage{natbib}
\usepackage{float}
\usepackage{multirow}

\usepackage{dsfont}
\usepackage{booktabs}

\usepackage{amsmath,amsfonts,bm}

\def\eqref#1{(\ref{#1})}

\newcommand{\bg}{{\mathbf g}}

\newcommand{\bx}{{\mathbf x}}

\newcommand{\bz}{{\mathbf z}}

\newcommand{\bA}{{\mathbf A}}
\newcommand{\bD}{{\mathbf D}}

\newcommand{\bG}{{\mathbf G}}
\newcommand{\bH}{{\mathbf H}}

\newcommand{\bL}{{\mathbf L}}

\newcommand{\bS}{{\mathbf S}}

\newcommand{\bN}{{\mathbf N}}

\newcommand{\bZ}{{\mathbf Z}}

\newcommand{\bX}{{\mathbf X}}

\newcommand{\bPhi}{{\boldsymbol \Phi}}
\newcommand{\bLambda}{{\boldsymbol \Lambda}}

\newcommand{\bTheta}{{\boldsymbol \Theta}}

\newcommand{\cB}{{\mathcal B}}

\newcommand{\cG}{{\mathcal G}}

\newcommand{\cS}{{\mathcal S}}

\newcommand{\cV}{{\mathcal V}}
\newcommand{\cW}{{\mathcal W}}

\def\bC{{\bf C}}

\def\bU{{\bf U}}

\def\minwrt[#1]{\underset{#1}{\text{minimize}}}
\def\maxwrt[#1]{\underset{#1}{\text{max}}}

\title{SlepNet: Spectral Subgraph Representation Learning for Neural Dynamics}

\author{%
  Siddharth Viswanath$^{\ast 1}$ \quad Rahul Singh$^{\ast 1,2}$ \quad Yanlei Zhang$^3$ \quad J. Adam Noah$^4$ \AND Joy Hirsch$^{2,4,6,8,9}$ \quad Smita Krishnaswamy$^{1,2,5,7}$ \\[0.5cm]
  $^1$Department of Computer Science, Yale University\; $^2$Wu Tsai Institute, Yale University \\ 
  $^3$Mila-Quebec AI Institute\; $^4$Department of Psychiatry, Yale University \\ 
  $^5$Department of Genetics, Yale School of Medicine\; $^6$Department of Neuroscience, Yale University\\
  $^7$Computational Biology and Bioinformatics Program, Yale University\\
  $^8$Department of Comparative Medicine, Yale University\\
  $^9$Department of Medical Physics and Biomedical Engineering, University College London\\
  [0.2cm]
  $^\ast$Equal Contribution \quad Correspondence: \texttt{smita.krishnaswamy@yale.edu}
}

\begin{document}

\maketitle

\begin{abstract}
Graph neural networks have been useful in machine learning on graph-structured data, particularly for node classification and some types of graph classification tasks. However, they have had limited use in representing patterning of signals over graphs. Patterning of signals over graphs and in subgraphs carries important information in many domains including neuroscience. Neural signals are spatiotemporally patterned, high dimensional and difficult to decode. Graph signal processing and associated GCN models utilize the graph Fourier transform and are unable to efficiently represent spatially or spectrally localized signal patterning on graphs. Wavelet transforms have shown promise here, but offer non-canonical representations and cannot be tightly confined to subgraphs. Here we propose SlepNet, a novel GCN architecture that uses Slepian bases rather than graph Fourier harmonics. In SlepNet, the Slepian harmonics optimally concentrate signal energy on specifically relevant subgraphs that are automatically learned with a mask. Thus, they can produce canonical and highly resolved representations of neural activity, focusing energy of harmonics on areas of the brain which are activated (i.e., visual cortex during movie watching). We evaluated SlepNet across three fMRI datasets, spanning cognitive and visual tasks, and two traffic dynamics datasets, comparing its performance against conventional GNNs and graph signal processing constructs. SlepNet outperforms the baselines in all datasets. Moreover, the extracted representations of signal patterns from SlepNet offers more resolution in distinguishing between similar patterns, and thus represent brain signaling transients as informative trajectories. Here we have shown that these extracted trajectory representations can be used for other downstream untrained tasks. Thus we establish that SlepNet is useful both for prediction and representation learning in spatiotemporal data.
\end{abstract}

\section{Introduction}
\label{sec:introduction}


Neural signals such as those measured in EEG, fMRI, or MEG exhibit complex spatiotemporally patterned structure. These signals form patterns over regions and subregions of the brain that evolve over time carrying information about cognition and function. Representing and decoding such signals necessitates a growing need for a neural network that can not only leverage structural information but also model how signals evolve and concentrate over specific subregions or regions. Graphs form a natural model for neural activity since they can model structural and functional connectivity. Researchers have therefore used graph neural networks in recent literature to analyze neural data~\cite{GadZhaPfe20,LuoWuYan24,MohKar24}. However message passing graph neural networks are designed mainly for deriving node representations, and certain kinds of graph representations. Very few works in the graph neural networks directly tackle the embedding of signals patterned over a graph. 

Graph signal processing~\cite{ShuNarFro13}, and associated graph convolutional networks like SpectralGCN~\cite{BruZarSzl14}, first order GCN~\cite{KipWel17} , BLISNet~\cite{xu2023blisnetclassifyinganalyzingsignals}, LEGS~\cite{Tong_LEGS}, offer a promising alternative. Existing GCN designs are based the graph Fourier transform (GFT)~\cite{IsuGamShu24}, which provides a principled framework for learning graph filters. However, since GFT is a global transform, it suffers from poor spatial localization and is not tuned to represent spatially resolved or localized phenomena within signals. Graph Wavelets~\cite{HamVan11,xu2023blisnetclassifyinganalyzingsignals} can capture localized signals, but often rely on non-canonical constructions that are not tuned to specific signal localization. Additionally, when analyzing a specific brain region- or subgraphs in general- graph wavelet-based methods suffer from information leakage at the boundaries of the subgraphs, where the signal energy intended to be localized within a subgraph spills over to its boundaries. These limitations highlight the need for a more principled approach to localized graph signal representations.

In order to address these limitations, we turn to graph Slepians~\cite{VilDemPre17, PetBolPre19}, which are alternative harmonics on a graph. Graph Slepians are functions that generalize classical Slepian functions~\cite{SlePol61,Sle78} to graph signal domains, incorporating the notions of subgraph selectivity and spectral bandwidth to reflect imposed constraints. The design of graph Slepians is based on an optimization criterion that expresses energy concentration of a band-limited graph signal in a predefined subgraph. Building on this foundation, we propose a novel GCN architecture called \textbf{SlepNet}, capable of learning both the subgraph and associated Slepian harmonics. Notably, SlepNet does not require subgraph of interest to be predefined, instead it learns both the relevant subgraph and the localized filters during training. This allows the model to focus on important functional regions of the graph, such as anatomically or functionally coherent regions of the brain, while also deriving rich representations of signals concentrated in those regions.  Moreover, to avoid the complexity of eigendecomposition for large graphs, we propose an efficient neural eigenmapping method.

We validate Slepnet using a 2-tiered evaluation system: 1) First, we train Slepnet on a classification task based on time series brain-activity data, and evaluate its classification accuracy as well as the accuracy of the learned mask, 2) Second, we extract the Slepian-based representation of the brain signals over time and evaluate the learned trajectory via visualization with T-PHATE as well as in the ability to perform an unrelated downstream classification using the highly resovled representation provided by the Slepian basis. For evaluation we use four datasets spanning neuroscience and real-world domains. The first two are fMRI datasets focused on obsessive compulsive disorder (OCD). One of them capture resting-state activity of participants with OCD and healthy controls after a perceptual and value-based decision making (PVDM) task and the other captures resting-state activity after a risk and ambiguity (RA) decision making task. The goal of our model on the two datasets is to predict if the subject has OCD or not. The third dataset is also an fMRI dataset that captures task-based activity while subjects with autism spectrum disorder (ASD) and healthy controls watched coherent and scrambled point light animations. The objective here is to perform a binary classification to determine whether the subject has ASD or not. Lastly, we used a real-world data set from the Caltrans Performance Measurement System (PeMS), which records 5-minute interval traffic measurements from road sensors across highways in California. The goal is to perform multi-class classification task to predict the day of the week from spatiotemporal traffic patterns. Our key contributions are:
\setlength{\leftmargini}{4pt}
\begin{itemize}
    \item \textbf{Slepian-based graph neural network:} We introduce a novel graph convolutional network architecture leveraging a key property of graph Slepian - optimally concentrate signal energy within a relevant learned subgraph while remaining spectrally bandlimited. This enables our model to learn important brain regions by focusing on signal representation within subgraphs.

    \item \textbf{Neural eigenmapping for scalable Slepian filter computation:} In order to address the computational cost of solving the constrained eigenvector problem at each step, we propose a neural eigenmapping method that learns to approximates the Slepian harmonics using supervised regression.

    \item \textbf{Validation of learned mask on ground truth data:} We show that the mask learned by SlepNet aligns with known embedded ground truth subgraphs, validating the model's ability to recover interpretable brain regions.

    \item \textbf{Two-tier evaluation on fMRI datasets:} We perform a two-tier evaluation of SlepNet across three fMRI datasets. In Tier 1, we assess the performance of subject-level classification, showing that SlepNet outperforms baselines. In Tier 2, we show that the embeddings extracted from SlepNet can be used for downstream analysis such as visualization and auxiliary classification, highlighting the expressivity and interpretability of the learned representations.
\end{itemize}

\section{Background}
\label{sec:background}

\subsection{Problem Definition}
\label{subsec:problem}
Let $\cG = (\cV,\bA)$ be a graph, where $\cV$ is the set of $N$ number of nodes, $\bA$ is the (weighted) adjacency matrix of the graph. The input node features are represented as a matrix $\bX \in \mathbb{R}^{N \times F} $ with $\bx_i \in \mathbb{R}^N$ (column of $\bX$) representing the $i^{th}$ feature channel of $\bX$ and $F$ denotes the total number of feature channels. The goal of graph convolutional networks (GCNs) is to learn expressive node representations in a data-driven manner. Conventional GCNs were rooted from graph Fourier transform (GFT) providing notion of frequency on graphs, which has been studied under the umbrella of graph signal processing.

\subsection{Graph Signal Processing and Graph Convolution Networks}
\label{subsec:GSP}
Graph signal processing (GSP)~\cite{ShuNarFro13, ManChaSin18, CheShiWri20} extends classical signal processing concepts and tools to graph-structured signals. GSP relates the vertex and spectral domains of a graph, much as classical signal processing connects the time and frequency domains of a time series. The eigenvalues and eigenvectors of the graph Laplacian provide a notion of frequency for signals defined on a graph. The graph Laplacian eigenvectors associated with low frequencies, vary slowly across the graph, i.e., if two vertices are connected by an edge, the values of the eigenvector at those locations are likely to be similar. The eigenvectors associated with larger eigenvalues oscillate more rapidly and are more likely to have dissimilar values on vertices connected by an edge. The graph Fourier transform and its inverse give us a way to equivalently represent a signal in two different domains: the vertex domain and the graph spectral domain. \\

The graph Fourier analysis relies on the spectral decomposition of graph Laplacians. The traditional combinatorial graph Laplacian is defined as $\bL = \bD - \bA$, with $\bD = \mathrm{diag}\{ d_1,d_2,\ldots, d_N \}$ and $d_i  = \sum_j A_{ij}$; the normalized version of the Laplacian is $\bL_{\mathrm{n}} = \bD^{-1/2} \bL \bD^{-1/2}$. The eigendecomposition of the graph Laplacian $\bL = \bU \bLambda \bU^T$, where $\bU \in \mathbb{R}^{N\times N}$ contains orthonormal eigenvectors as its columns and $\bLambda = \mathrm{diag}\{\lambda_1, \ldots, \lambda_N\}$ is a diagonal matrix of eigenvalues, is used to define graph Fourier transform (GFT) with eigenvectors of the graph Laplacian being the graph Fourier harmonics and the corresponding eigenvalues being the graph frequencies~\cite{ShuNarFro13}. With $\lambda_1 \leq \lambda_2 \leq \ldots \leq \lambda_N$, where $\lambda_1$ corresponds to the lowest (zero) frequency and $\lambda_N$ corresponds to the highest frequency of the graph. When using normalized Laplacian $\bL_{\mathrm{n}}$, the graph frequencies $\lambda_\ell \in [0,2]$, with $\lambda_1 = 0$.  

Let $\bx \in \mathbb{R}^{N}$ be a single-channel input graph signal, the GFT and the inverse GFT are defined as $\hat{\bx} =  \bU^T \bx$ and $\bx = \bU \hat{\bx}$, respectively.  Graph convolution of an input graph signal $\bx$ with a filter $\bg$ is 
\begin{equation}\label{eq:filter1}
    \bx * \bg := \bU \left( (\bU^T \bx) \odot (\bU^T \bg) \right) = \bU \hat{\bG} \bU^T \bx, 
\end{equation}
where $\hat{\bG} := \mathrm{diag} (\hat\bg)=\mathrm{diag}\{\hat{g}_1,  \ldots, \hat{g}_N\}$, and $\odot$ denotes element-wise multiplication.

The graph convolution operation given by \eqref{eq:filter1} is used in GNNs to learn filter coefficients in spectral domain. Spectral-GNN~\cite{BruZarSzl14} learns all the $N$ number of filter coefficients. For computational efficiency, the filter coefficients can be approximated via $M^{th}$ order polynomials of the graph frequencies ($M<<N$), which allows to write graph convolution in spatial domain via polynomials in the Laplacian. ChebNet~\cite{DefBreXav16} proposed to approximate the graph convolution via $M^{th}$ order Chebyshev polynomials. \citet{KipWel17} further simplified it by assuming first order polynomial filter. 


\begin{figure}[t]
\vspace{-0.1cm}
\centering
\begin{subfigure}[t]{0.45\textwidth}
\centering	
\includegraphics[scale=0.23]{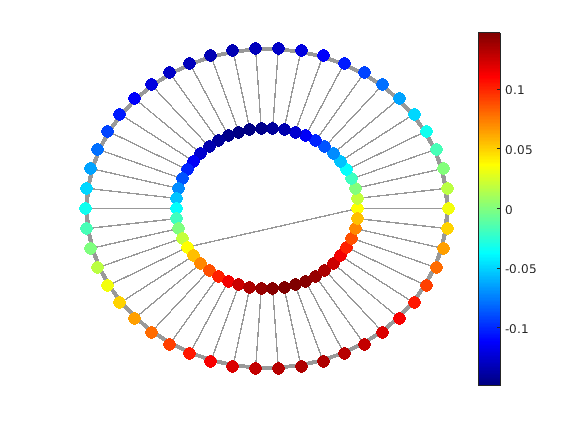}\includegraphics[scale=0.23]{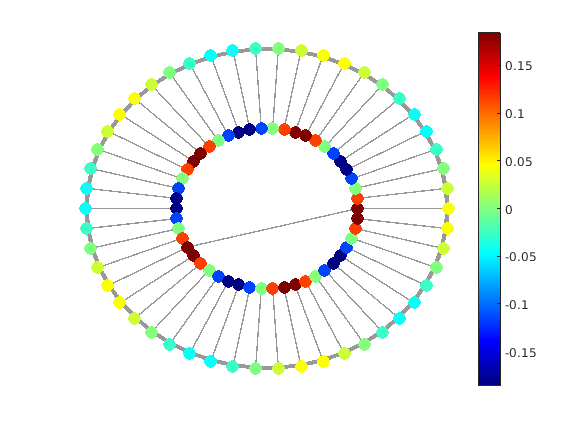}
\caption{}
\end{subfigure}
\hspace{1cm}
\begin{subfigure}[t]{0.45\textwidth}
\centering	
\includegraphics[scale=0.23]{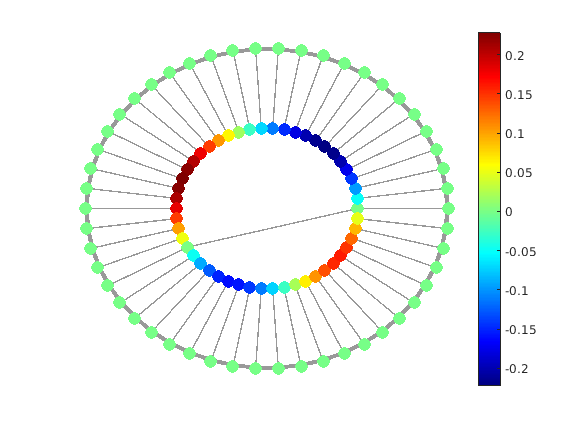}\includegraphics[scale=0.23]{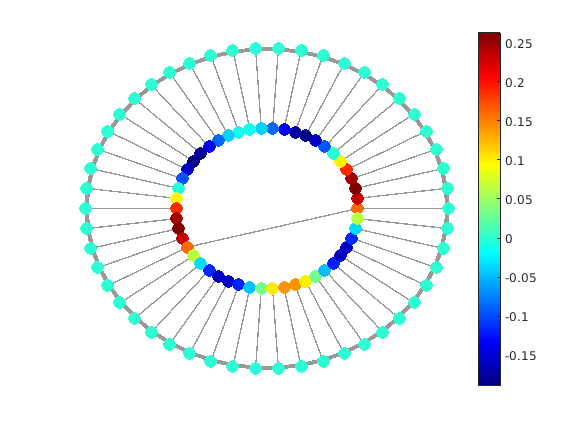}
\caption{}
\end{subfigure}
\caption{(a) Example graph Fourier harmonics (b) Graph Slepian harmonics. Here the subgraph we are interested in is the inner circle: Graph Slepian harmonics are concentrated within the subgraph while Fourier harmonics are global.}
\label{fig:toy1}
\end{figure}



\subsection{Graph Slepians}
\label{subsec:slepians}
Graph Slepians~\cite{VilDemPre17} are the extension of classical Slepians~\cite{SlePol61, Sle78}, which solves the fundamental problem of optimally concentrating a signal jointly in spatial and graph spectral domains. Suppose we are interested in a specified set of nodes $\cS \subset \cV$ that contains $N_{S}$ nodes in which we want the energy concentration to be maximal. Graph Slepians provides a way to design harmonics that are maximally concentrated within a predefined subset of nodes while also ensuring that they are bandlimited to a certain frequency. Specifically, it involves finding harmonics that are bandlimited to the first $K$ graph placian eigenvectors and  having maximum energy concentration over a subset of nodes $\cS \subset \cV$, i.e.,
\begin{align}
\label{eq:slep_problem}
    \underset{\bz}{\mathrm{maximize}} \quad \sum_{i \in \cS}~[z(i)]^2 & \quad \mathrm{subject~to} \quad \hat{z}(\lambda_\ell) = 0~~ \mathrm{for}~~ \lambda_\ell > K.
\end{align}
Let $\bS_{\mathrm{B}}$ be the diagonal matrix with first $K$ entries to be 1, and $\bS_{\mathrm{V}}$ be the diagonal node selection matrix
with its diagonal entries be 1 corresponding to the selected node index and zero otherwise:
\begin{align*}
    \bS_{\mathrm{B}} = \begin{bmatrix}
        1 & 0 & 0 & 0 & 0 \\
        0 & \ddots & 0 & 0 & 0 \\
        0 & 0 & 1 & 0 & 0 \\
        0 & 0 & 0 & 0 & 0 \\
        0 & 0 & 0 & 0 & \ddots \\
    \end{bmatrix}, \quad \bS_{\mathrm{V}} = \begin{bmatrix}
        0 & 0 & 0 & 0 & 0 \\
        0 & 1 & 0 & 0 & 0 \\
        0 & 0 & \ddots & 0 & 0 \\
        0 & 0 & 0 & 0 & 0 \\
        0 & 0 & 0 & 0 & 1 \\
    \end{bmatrix}.
\end{align*}
Then the problem \eqref{eq:slep_problem} is to find the vector that maximizes the \textbf{energy concentration criterion}
\begin{equation}
\label{eq:mu}
    \underset{\hat{\bz}}{\mathrm{argmax}}\quad  \frac{\hat{\bz}^T \bS_{\mathrm{B}}^T \bU^T \bS_{\mathrm{V}} \bU \bS_{\mathrm{B}} \hat{\bz}}{\hat{\bz}^T \hat{\bz}} = \frac{\hat{\bz}^T \bC \hat{\bz}}{\hat{\bz}^T \hat{\bz}},
\end{equation}
where
\begin{equation}
\label{eq:energy_concen}
    \bC = \bS_{\mathrm{B}}^T \bU^T \bS_{\mathrm{V}} \bU \bS_{\mathrm{B}}
\end{equation}
and $\bU$ is the eigenvector matrix of the graph Laplacian. The problem reverts to finding the eigenvectors of the concentration matrix $\bC$. Graph Slepians are subsequently calculated as $\bz_k = \bU \hat{\bz}_k,$ for $k = 1, \ldots, K$, where $\hat{\bz}_k$ are the eigenvectors of $\bC$. Note that the Slepian harmonics are orthonormal over the entire graph as well as orthogonal over the subset $\cS$:  $\bz_i^T \bz_j = \delta_{i-j}$ and  $\bz_i^T \bS_{\mathrm{V}} \bz_j =  \delta_{i-j}$. An example illustration of comparing graph Fourier and Slepians harmonics is shown in Figure~\ref{fig:toy1}. 

In additional to the energy concentration criterion, there exists another graph Slepian design that optimizes the \textbf{modified embedded distance criterion}~\cite{VilDemPre17}
\begin{equation}
\label{eq:xi}
    \underset{\hat{\tilde{\bz}}}{\mathrm{argmin}}\quad  \frac{\hat{\tilde{\bz}}^T \bLambda_K^{1/2} \bS_{\mathrm{B}}^T \bU^T \bS_{\mathrm{V}} \bU \bS_{\mathrm{B}}  \bLambda_K^{1/2} \hat{\tilde{\bz}}}{\hat{\tilde{\bz}}^T \hat{\tilde{\bz}}} = \frac{\hat{\tilde{\bz}}^T \bC_{\mathrm{emb}} \hat{\tilde{\bz}}}{\hat{\tilde{\bz}}^T \hat{\tilde{\bz}}},
\end{equation}
where $\bLambda_K = \bS_{\mathrm{B}}^T \bLambda \bS_{\mathrm{B}}$ and the modified concentration matrix $\bC_{\mathrm{emb}} = \bLambda_K^{1/2} \bC \bLambda_K^{1/2}$. Slepians based on the above criterion are then computed as $\tilde{\bz}_k = \bU \hat{\tilde{\bz}}_k,$ for $k = 1, \ldots, K$, where $\hat{\tilde \bz}_k$ are the eigenvectors of $\bC_{\mathrm{emb}}$.

Both Slepian designs exhibit vertex-domain localization and spectral bandlimit-ness. However, a key distinction arises in their interpretation: the eigenvalues of $\bC_{\mathrm{emb}}$ represent the modified embedded distance, effectively a measure of ``frequency" localized within the subgraph $\cS$, whereas the eigenvalues of $\bC$ represent the energy concentration in the subgraph. This highlights their key difference in smoothness, i.e., Slepian signals based on $\bC_{\mathrm{emb}}$ are constructed to be both smooth and localized, while Slepian signals based on $\bC$ prioritize local energy concentration irrespective of its oscillatory content.
\begin{figure}[t]
\centering
\includegraphics[width=0.9\linewidth]{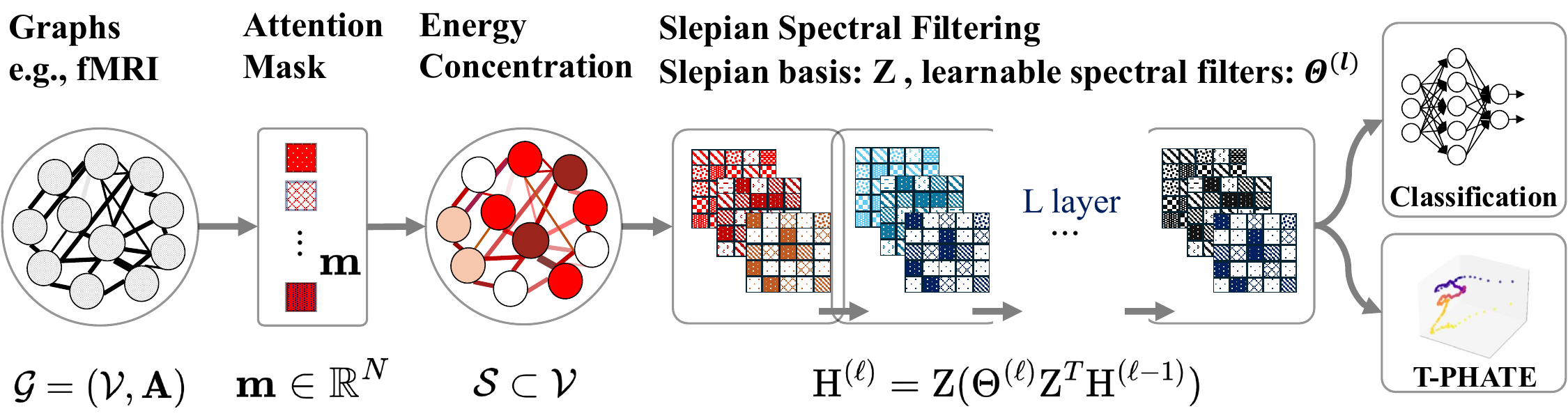}
\caption{Architecture of SlepNet.}
\label{fig:slepnet_arch}
\end{figure}

\section{Theoretical Motivation}
\label{sec:theory}

Here we show that the graph Slepians can also be interpreted as graph signals that are constrained to a subset of nodes $\mathcal{S} \subset \mathcal{V}$, such that it achieves the maximum energy concentration within a specific frequency band $\mathcal{W} \subset \mathcal{B}$. We formalize this in following proposition :
\begin{prop}
\label{prop:1}
The graph Slepian harmonics spacelimited to a set of nodes $\cS$ and maximizing energy in a certain frequency band $\cB$ are given by
    \begin{equation}
        \bz = \bS_{\mathrm{V}} \bU \hat{\bz},
    \end{equation}
    where $\hat{\bz}$ are the eigenvectors of the concentration matrix $\bC = \bS_{\mathrm{B}}^T \bU^T \bS_{\mathrm{V}} \bU \bS_{\mathrm{B}}$.
\end{prop}
\begin{proof}
    See Appendix~\ref{subsec:proof}.
\end{proof}

Note that by definition, graph Slepians are designed to solve the optimization problem \eqref{eq:slep_problem} where energy is maximized within a fixed subset of nodes while enforcing bandlimit-ness. In contrast, Proposition~\ref{prop:1}, provides an alternative interpretation of maximizing energy in a frequency band while enforcing spacelimit-ness. For neural signals, we can maximize energy in the fixed energy band while selecting specific brain areas. For example, visual tasks will likely information in the visual cortex and related regions but in bandwidths limited by noise in measurements. Thus, this motivates our use of Slepians for brain signals. 

\section{Methodology}
\label{sec:main}

We propose SlepNet that incorporates Slepian harmonics into spectral graph convolutions allowing us to capture both local and global information in the graph. Figure~\ref{fig:slepnet_arch} illustrates the proposed SlepNet architecture, which contains two modules: subgraph mask learning followed by Slepian convolutions.

\subsection{Attention-based Mask Learning Module}


In this module, we design an \textbf{attention-based} mechanism to learn a mask $\mathbf{m} \in \mathbb{R}^N$, which identifies a subset of nodes $\mathcal{S} \subset \mathcal{V}$ in the input graph that are the most relevant for the downstream task. The goal of the mask is to select selection matrix $\bS_{\mathrm{V}}$, for computation of the Slepian basis. While the Slepian designs do not impose any requirements for the selected of nodes to be spatially contiguous, promoting such contiguity is often desirable to capture coherent regions of the brain in neural signaling, improve model interpretability, or introduce structural priors. To address this, we first cluster the graph nodes using spectral clustering ~\cite{NgJorWei01,ManChaSin18}. Then, we apply the mask selection process on clusters of graph nodes rather than the individual nodes themselves. 

Let $\mathbf{M} \in \mathbb{R}^{ \kappa \times N}$ be a one-hot cluster assignment matrix, where each row $M_i$ corresponds to a single cluster and contains ones at the positions of the nodes belonging to that cluster. Here, $\kappa$ is the chosen number of clusters for the spectral clustering. 
Then we compute cluster-level attention weights $Att = \{a_1, a_2, \ldots a_{\kappa} \}$ as follows:  
\begin{equation*}
  a_i=\sigma(w_i),
\end{equation*}
where $w_i$ are learnable weights and $\sigma$ is a sigmoidal non-linearity, which forces the value to be $0$ or $1$. In practice, this could also be the \textit{hard Tanh} activation in PyTorch. 

Since this attention changes with input features, is well-suited for time series data on graphs, as it allows the model to dynamically focus on regions that vary across time. The resulting selected regions serve as supports over which localized Slepian spectral filtering is applied, allowing the model to focus its computation and representation learning on meaningful subgraphs. By constraining the subgraph selection process in this way, the learned mask not only retains the flexibility needed to capture task-relevant information but also benefits from improved interpretability and robustness.




\subsection{Slepian-based Spectral Filtering} 
Given the subset of nodes selected by the mask-learning module, we compute the corresponding Slepian harmonics of the graph, which serve as a localized spectral basis for the selected region(s). These harmonics are then used to perform spectral filtering on the input graph signals. Specifically, for an input graph signal $\bx$, the Slepian-filtered output is given by 
\begin{equation}
    \tilde{\bx} = \bZ(\bTheta \bZ^T \bx),
\end{equation}
where $\bZ$ denotes the matrix of Slepian harmonics (restricted to the selected node subset) and  $\bTheta$ is the diagonal filter response matrix.

Combining the mask learning module and Slepian spectral filtering, SlepNet jointly learns structured subgraph masks and corresponding spectral filters. The output of the $\ell^{th}$ layer of SlepNet is: 
\begin{equation}\label{eq:slepian}
    \bH_{[:,j]}^{(\ell)} =  \sigma \left( \bZ \sum_{i=1}^p \bTheta_{i,j}^{(\ell)} \bZ^T \bH_{[:,i]}^{(\ell - 1)} \right), \quad j=1,\ldots,q \quad \mathrm{and} \quad \bH^{(0)} = \bX. 
\end{equation}
Here, $\bTheta_{i,j}^{(\ell)}$ is a diagonal filter matrix learned in the Slepian spectral domain, $\bZ$ is a matrix containing Slepian harmonics as its columns, $\bH_{[:,i]}$ denotes the $i^{th}$ column of $\bH$, and $\sigma$ is an element-wise non-linear activation function such as ReLU. This SlepNet layer transforms the input feature matrix $\bH^{(\ell - 1)} \in \mathbb{R}^{N \times p}$ into an output feature matrix $\bH^{(\ell)} \in \mathbb{R}^{N \times q}$. After $L$ layers, the output features $\bH^{(L)}$ are fed to a task-specific head (e.g., an MLP classifier for graph-level prediction). Depending on the type of Slepian harmonics employed, we refer to our model as SlepNet-I when using Slepians maximizing \eqref{eq:mu}, and SlepNet-II when using Slepians optimizing \eqref{eq:xi}, for spectral filtering. 



Since the parameter update of the model requires eigendecomposition of the energy concentration matrix (or modified concentration matrix) in each backpropagation step, we ensure differentiability by relying on established gradients of eigenvalues and eigenvectors for symmetric matrices~\cite{7410696}. 
The gradients rely on the matrix being symmetric to guarantee real eigenvalues. However, because gradient of the eigenvectors become unstable when eigenvalues are too close or not distinct, we add a slight perturbation to the matrix for numerical stability~\cite{LiShiZha24}.

\subsubsection{Approximate Slepian Spectral Filtering}

\begin{wrapfigure}{r}{0.45\textwidth}
\centering
\includegraphics[width=0.99\linewidth]{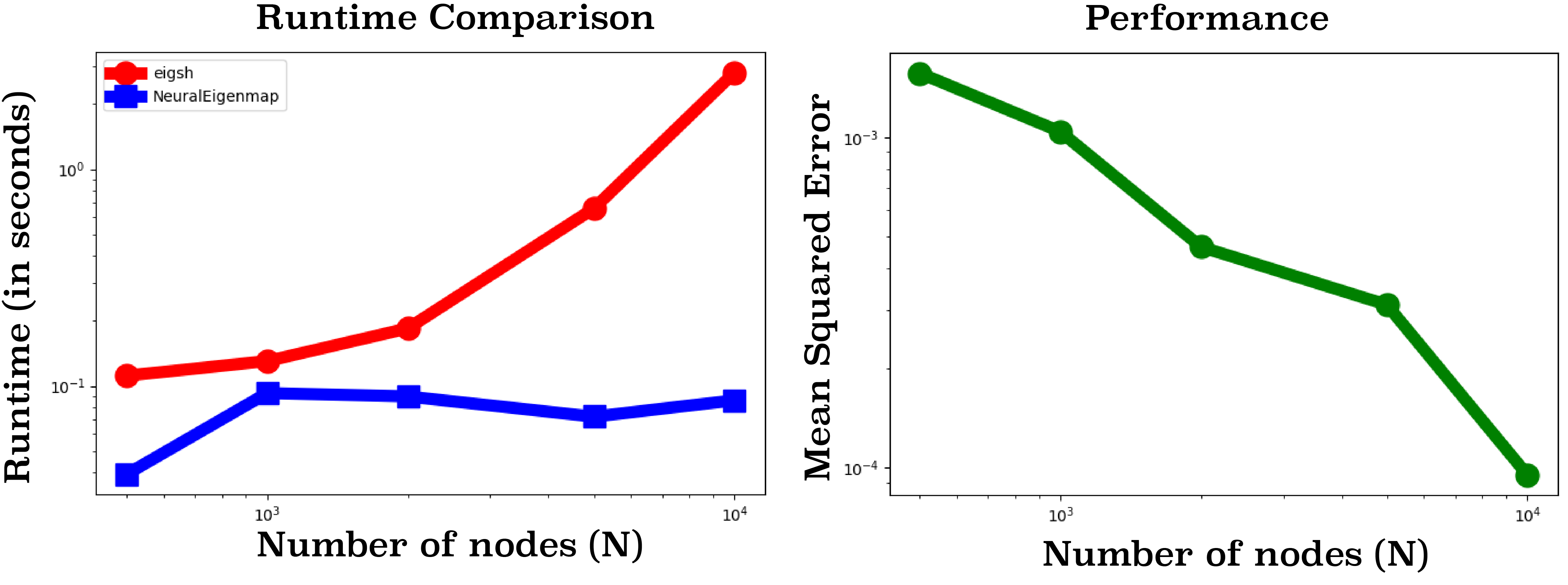}
\caption{(Left) Comparison of runtimes for eigenmapping neural network vs full eigendecomposition. (Right) Performance of Neural eigenmap predictions of Slepians on unseen nodes.}
\label{fig:run_time}
\end{wrapfigure}
Eigendecomposition has complexity $O(N^3)$ where $N$ is the number of vertices in the graph. In order to alleviate this, we invoke the recently developed concept of neural eigenmapping \cite{shaham2018spectralnet, mishne2019diffusion, duque2022geometry}. Neural eigenmapping methods effectively require the computation of eigenvectors on a subset of nodes in the graph for training purposes. In our design, the Slepian eigenmapping neural network $\bN$ takes a $1$-hot node identifier as input and outputs the Slepian eigencoordinates of the vertex. Let $\mathds{1}_j(i)$ be the indicator vector 
\[
\mathds{1}_j(i) = 
\begin{cases}
1, & \text{if } i = j \\
0, & \text{otherwise}
\end{cases}
\quad \text{for } i,j \in \{1, \dots, N\}.
\]
Then the neural network $\bN: \mathbb{R}^N \rightarrow \mathbb{R}^K$, where $K$ is the number of eigenvectors we are recovering for the Slepian basis:
\begin{equation}
   \bN (\bPhi, \mathds{1}(i)) = [z_1(i), \ldots, z_k(i), \ldots, z_K(i)]^T.
\end{equation}
Therefore, when $\mathds{1}(i)$ is the input to the network, it outputs the Slepian eigenvector entries at node $i$: $[z_1(i), \ldots, z_k(i), \ldots, z_K(i)]^T$, where $\bz_k$ is the $k^{th}$ Slepian eigenvector. Once trained, the neural network can be used to efficiently extend the eigenvectors to the remaining vertices.  
Note that similar approaches have been proposed in the realm of non-linear dimensionality reduction. Spectralnet performs general eigenmapping \cite{shaham2018spectralnet}, DiffusionNet maps datapoints to diffusion map coordinates \cite{mishne2019diffusion}, GRAE (geometry regularized autoencoders) maps points to PHATE~\cite{MooVanWan19} coordinates. However, to our knowledge, this is the first application of this technique for graph spectral analysis. As shown in Figure~\ref{fig:run_time}, the runtime of classical eigendecomposition increases as the graph size increases, exhibiting poor scalability beyond 1000 nodes. In contrast, the eigenmapping neural network model has significantly lower and more stable run-times, demonstrating its computational advantage when dealing with large-scale graphs. 
\section{Experiments}
\label{sec:experiments}
We evaluate SlepNet\footnote{The code is available at \url{https://github.com/KrishnaswamyLab/SlepNet}.} in several ways (1) We assess the accuracy of the learned mask,  (2) We assess the overall classification accuracy of SlepNet on neuronal prediction tasks.  This includes mask learning as well as subsequent representation layers, (3) We extract the Slepian-basis representation of signals from SlepNet, visualize these brain transients, and utilize these for a subsequent more fine-grained classification in Section~\ref{subsec:representations}. These evaluations are conducted on three fMRI datasets: two focused on obsessive compulsive disorder (OCD) and one on autism spectrum disorder (ASD). The datasets are thoroughly described in Appendix~\ref{subsec:dataset}. The OCD datasets are represented as anatomically-informed brain graphs with each node corresponding to a brain region with associated temporal signals. The ASD dataset is modeled as functional connectivity graphs where the edges represent Pearson correlation between the time series. Additionally, to showcase applicability beyond neuroscience, we evaluate SlepNet a dataset derived from the Caltrans Performance Measurement System (PeMS), where the nodes correspond to road sensors and the task is to predict the day of the week based on traffic dynamics. Hyperparameters and experimental details are described in Appendix~\ref{subsec:exp_details}.

\subsection{Interpretable Subgraph Selection via Learned Mask}
\label{subsec:mask_results}
\begin{wrapfigure}{r}{0.56\textwidth}
\centering
\vspace{-0.5cm}
\includegraphics[width=1.0\linewidth]{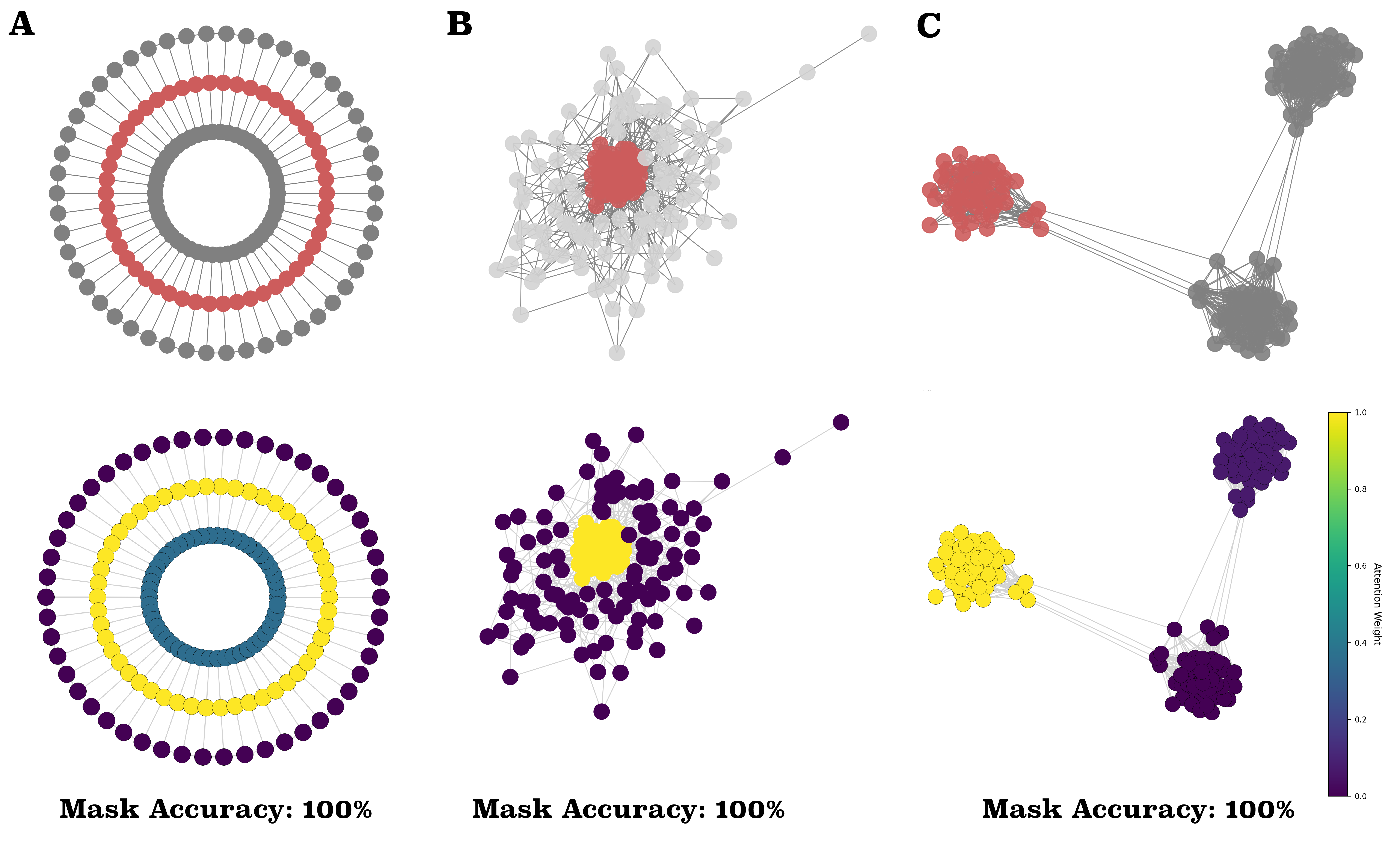}
\caption{Visualization of learned mask for three synthetic datasets. \textbf{Top Row:} Ground truth where the subgraph containing informative signals is highlighted in red. \textbf{Bottom Row:} Learned mask weights from SlepNet, visualized using a continuous colormap (yellow=high weight, purple=low weight). (A) Three-ring graph with middle ring as the informative subgraph, (B)Three-component connected ER graph with one component as the informative subgraph, (C) A single ER graph with connected subgraph embedded near the center.}
\label{fig:toy_mask}
\end{wrapfigure}
In order to assess SlepNet's ability identify relevant subgraphs, we design three synthetic graph datasets with graph-feature variations concentrated in subgraphs. Each graph consists of 150 nodes with associated 10 dimensional temporal signals. We simulated three different types of graphs: (1) A three-ring graph with 50 nodes per ring, (2) A three-component Erd\H{o}s-R\'enyi (ER) graph with intra-component connectivity, and (3) a single ER graph with global connectivity and a embedded dense subgraph. In all the datasets, signal information was concentrated on the subgraphs while the remaining nodes contained Gaussian noise as signals. We evaluate whether the mask learned by SlepNet accurately highlights the ground truth subgraphs as shown in Figure~\ref{fig:toy_mask} and find that it sucessfully recovers the ground truth subgraphs with \textbf{100\%} accuracy. Additionally, we analyze the learned mask that are optimized during training of the PVDM and RA fMRI datasets to identify spatially localized brain regions that contribute most towards the classification. As seen in Figure \ref{fig:mask} (see Appendix), we project the learned mask weights from SlepNet onto the brain for both datasets.

\subsection{SlepNet Outperforms other Methods in Classification}
\label{subsec:class_results}

We first evaluate the performance of our SlepNet model on classification tasks. In the two OCD datasets, the model was trained to classify between OCD and HC subjects, while in the ASD-ABIDE dataset, the task was to classify between ASD and HC subjects. Each brain region (node) is associated with a resting-state temporal fMRI signal. For each subject, we process the time series by taking snapshots per timepoint, hence generating one graph per timepoint along with sinusoidal embeddings to account for temporal position. This allows the model to learn time-aware representations via the sinusoidal embeddings. The model then learns node level representations using Slepians localized to relevant subgraphs and this is then passed into a classifier for binary classification outcome. 

We compare the classification performance to a diverse set of baselines, including Graph Neural Networks (Spectral GCN~\cite{BruZarSzl14}, GCN~\cite{KipWel17}, GAT~\cite{veličković2018graphattentionnetworks}, GIN~\cite{xu2019powerfulgraphneuralnetworks}, and GraphSAGE~\cite{NIPS2017_5dd9db5e}) and the Graph Wavelet Transform. In Table \ref{tab:ocd_five_dataset}, we show that SlepNet achieves the top result on both datasets demonstrating its ability to model brain graphs and extract temporally relevant features for subject-level classification. A bolded score and underline score indicate highest-overall and second-best performance, respectively. Details of an ablation study of the model are provided in Appendix~\ref{sec:ablations}.
\begin{table}[ht]
\centering
\caption{\footnotesize Classification accuracy (\%) across five benchmark datasets: two OCD decision-making tasks (PVDM and RA), an ASD perception task, and two traffic forecasting datasets (PEMS03 and PEMS07). Binary classification tasks have $c=2$ classes; traffic datasets involve day-of-week prediction with $c=7$ classes.}
{\renewcommand{\arraystretch}{1.2}
\resizebox{\textwidth}{!}{\begin{tabular}{lccccc}
\toprule
\textbf{Model} & 
\textbf{OCD-PVDM} (c=2) & 
\textbf{OCD-RA} (c=2) & 
\textbf{ASD-ABIDE} (c=2) & 
\textbf{Traffic-PEMS03} (c=7) & 
\textbf{Traffic-PEMS07} (c=7) \\
\hline
Spectral GCN & 62.62~$\pm$ 0.48\% & 69.76~$\pm$ 0.73\% & 57.68~$\pm$ 0.59\% & \underline{56.31~$\pm$ 0.66\%} & \underline{68.48~$\pm$ 0.65\%} \\
Graph Wavelets & 55.31~$\pm$ 0.77\% & 57.86~$\pm$ 0.79\% & 51.61~$\pm$ 0.56\% & 18.95~$\pm$ 0.69\% & 19.14~$\pm$ 0.62\% \\
GCN & 55.01~$\pm$ 0.80\% & 57.65~$\pm$ 0.85\% & 51.17~$\pm$ 0.70\% & 17.65~$\pm$ 0.71\% & 18.36~$\pm$ 1.01\% \\
GAT & 55.01~$\pm$ 0.80\% & 57.65~$\pm$ 0.85\% & 51.14~$\pm$ 0.70\% & 17.93~$\pm$ 0.51\% & 18.74~$\pm$ 1.13\% \\
GIN & 58.13~$\pm$ 0.47\% & 60.09~$\pm$ 0.70\% & 51.51~$\pm$ 0.61\% & 16.42~$\pm$ 0.76\% & 16.02~$\pm$ 0.90\% \\
GraphSAGE & 55.01~$\pm$ 0.80\% & 57.65~$\pm$ 0.85\% & 51.14~$\pm$ 0.72\% & 24.14~$\pm$ 0.79\% & 30.58~$\pm$ 1.13\% \\
SlepNet - I : Energy Concentration & \textbf{84.70~$\pm$ 0.46\%} & \textbf{90.65~$\pm$ 0.62\%} & \textbf{74.13~$\pm$ 0.81\%} & 52.87~$\pm$ 0.94\% & 64.11~$\pm$ 1.54\% \\
SlepNet - II : Embedded Distance & \underline{74.55~$\pm$ 0.86\%} & \underline{81.10~$\pm$ 1.36\%} & \underline{56.89~$\pm$ 0.97\%} & \textbf{56.56~$\pm$ 0.76\%} & \textbf{70.16~$\pm$ 1.10\%} \\
\bottomrule
\end{tabular}}
\label{tab:ocd_five_dataset}
}
\end{table}
\subsection{SlepNet Learns Rich and Informative Embeddings}
\label{subsec:representations}
SlepNet learns informative and temporally-aware embeddings of brain dynamics. We evaluate these embeddings across three different aspects: (1) Trajectory visualization of the embeddings using T-PHATE~\cite{Busch2023} (2) Downstream classification performance (3) Curvature analysis of the trajectories. Together, the three evaluations provide a strong analysis of the embeddings learned by SlepNet.

\begin{wraptable}{r}{6.0cm}
\centering
\caption{\scriptsize Average curvature on PVDM and RA datasets}
\tiny
{\renewcommand{\arraystretch}{1.2}
\begin{tabular}{lcc}
\toprule
\textbf{Model} & \textbf{OCD-PVDM} & \textbf{OCD-RA} \\
\midrule
\tiny
Spectral GCN & 35.74~$\pm$ 1.54 & \underline{72.06~$\pm$ 3.44} \\
T-PHATE & 33.84~$\pm$ 1.15 & 33.48~$\pm$ 1.08 \\
SlepNet - I  & \textbf{73.79~$\pm$ 4.36} & 64.49~$\pm$ 2.73 \\
SlepNet - II  & \underline{70.17~$\pm$ 2.92} & \textbf{78.08~$\pm$ 4.19} \\
\bottomrule
\end{tabular}
\label{tab:curvature_across_models}
}
\end{wraptable}

\subsubsection{Trajectory Visualization and Curvature Analysis of SlepNet Trajectories}
\label{subsubsec:visualizations}

In order to evaluate the structure and interpretiblity of the SlepNet embeddings, we visualize them using T-PHATE, a nonlinear manifold learning dimensionality reduction method designed to preserve temporal trajectories on high dimensional data. We additionally compare them to the embeddings from Spectral GCN~\cite{BruZarSzl14} and applying T-PHATE to the raw fMRI time series. In Figure~\ref{fig:tphate_pvdm}, each subplot represents the trajectory of a subject colored temporally. SlepNet produces temporally coherent and structured embeddings corresponding to neural-state changes. On the other hand, Spectral GCN embeddings are temporally discontinuous and less structured hence exhibiting no patterns. Furthermore, the direct T-PHATE embeddings on the fMRI time series produce oversmoothed trajectories, whereas SlepNet embeddings are curvier and have more variation.  The oversmooth trajectories can fail distinguish fine grained differences in the neural dynamics, which is reflected in the higher secondary classification accuracy scores described below. Additional visualizations for the OCD-RA and ASD-ABIDE datasets are provided in the Appendix~\ref{subsec:viz_add}. 

\begin{figure}[ht]
\centering
\includegraphics[width=0.95\linewidth]{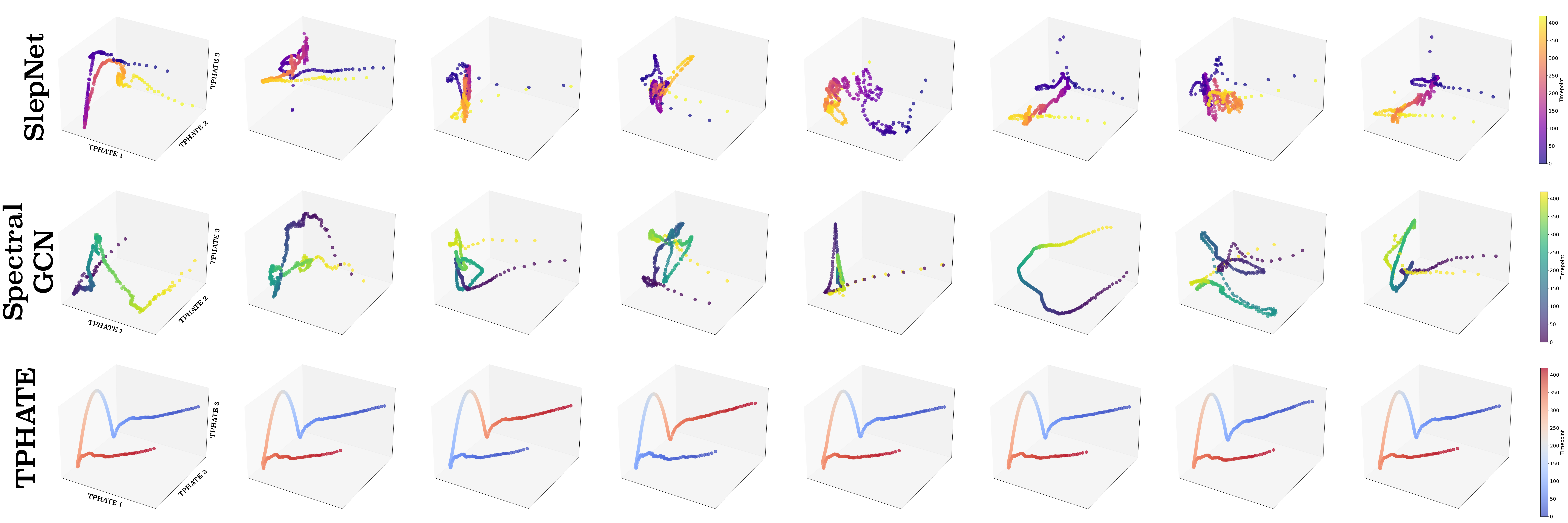}
\caption{T-PHATE visualization (PVDM dataset) of Graph Slepian embeddings using all nodes colored by timepoint, showing dynamic trajectories.}
\label{fig:tphate_pvdm}
\end{figure}

To better characterize the structure of the neural dynamics encoded in the trajectories, we compute their curvature. Curvature serves as a geometric feature that explains how rapidly the trajectory bends in latent space and can reflect transitions and bifurcations in the underlying neural state (see \citep{singh2025}).
For a 1-D trajectory $z(t)$ in 3D-Euclidean space, the curvature of $z(t)$ at point $p$ quantifies the intensity of curve turns from its original direction. Mathematically, this is given as the second order derivative of the trajectory. In the case of the circle defined by $z(t) = (x, y) = (r\cos(t), r\sin(t))$, $\tau = 1/r$ (see \cite{kreyszig1991}). In the setting where the trajectories are embedded from T-PHATE and formed in terms of discrete points, we use the local frame to fit a planar curve at the query point to estimate the curvature. More precisely, at a given point $p$, we select a neighborhood centered around $p$ of size $\approx 5\%$ of the length of the trajectory as our chosen samples for the evaluation of the local frame. Then the Singular Value Decomposition (SVD) is performed to obtain the local frame and local plane. Next the points from the neighborhood are projected to this local plane to fit a circle using the least square method. Subsequently, the radius $r$ of this fitted circle at point $p$ gives us the curvature $\tau = 1/r$. This process can be repeated throughout all the trajectory giving a complete curvature profile of the trajectory. 
To demonstrate complexity of SlepNet's trajectories, we compute the average curvature of each trajectory over all points within.  As shown in Table~\ref{tab:curvature_across_models}, SlepNet has the highest curvature values on average compare to other models, indicating that SlepNet's representation has more detailed and complex dynamics. 

\subsubsection{Downstream Classification from SlepNet Embeddings}

\begin{wraptable}{r}{4.4cm}
\centering
\tiny
\vspace{-1cm}
\caption{\scriptsize Classification of Male vs Female from T-PHATE embeddings}
{\renewcommand{\arraystretch}{1.2}
\begin{tabular}{lc}
\toprule
\textbf{Model} & \textbf{Classification Accuracy} \\
\hline
Spectral GCN & 53.62~$\pm$ 0.44\% \\
GCN & 54.38~$\pm$ 0.67\% \\
GAT & 53.43~$\pm$ 0.19\% \\
GIN & 55.71~$\pm$ 0.30\% \\
GraphSAGE & 53.57~$\pm$ 0.38\% \\
SlepNet - I  & \textbf{58.24~$\pm$ 1.40\%} \\
SlepNet - II  & \underline{55.10~$\pm$ 1.17\%} \\
\bottomrule
\end{tabular}
\label{tab:model_accuracy_3}}
\end{wraptable}
Finally, in order to quantify how informative SlepNet's embeddings are, we define a Tier 2 classification task on only the learned embeddings to further predict the sex of the subject (Male or Female). For this task, we used one of the OCD datasets, which contains metadata of the subject sex. We train a multi-layer perceptron (MLP) classifier using the model's learned embeddings as input and compute binary classification accuracy. As we can see in Table~\ref{tab:model_accuracy_3}, the embeddings from SlepNet yield higher accuracy than those obtained from the baselines. 





\section{Conclusions}
\label{sec:conclusions}

Here we presented SlepNet, a novel graph convolutional network that uses Slepian harmonics in place of graph Fourier harmonics for learning graph filters. Slepians have the advantage that they concentrate energy of the harmonics within selected frequency bands and vertices. SlepNet is a two-part network where the first part actually learns the vertex selectivity of the harmonics on the fly using a learnable mask. The second part computes filtered signals based on Slepian harmonics. We test SlepNet on 2-tiered task, where first SlepNet is trained to classify brain activity using relevant labels (such as disease), second we extract the Slepian-filtered representations and utilize it to visualize brain transients, as well as secondary, more fine-grained classification based on these representations. The visualization shows that trajectories extracted from SlepNet have more fine-grained details and curvature than other representations. Further, we show that utilization of this representation leads to more accurate downstream classification. 
\textbf{Limitations:}
We showcased a method for learning vertex selectivity of Slepian Harmonics, but we fixed the band selectivity. This too could be learnable. Additionally, while we use SlepNet here for characterization and classification of brain dynamics, we do not generate or learn models of the dynamics. This could be a future line of research. 
\textbf{Broader Impacts:}
SlepNet can have a positive societal impact by allowing the classification of brain trajectories to predict psychiatric conditions of patients, and decoding stimulus information. 

\section{Acknowledgments}
S.K. is funded in part by the NIH (NIGMSR01GM135929, R01GM130847), NSF CAREER award IIS-2047856, NSF IIS-2403317, NSF DMS-2327211 and NSF CISE-2403317. S.K is also funded by the Sloan Fellowship FG-2021-15883, the Novo Nordisk grant GR112933. R.S. is funded by the Wu Tsai Postdoctoral Fellowship from Yale University. J.H. is funded by NIH grants (NIMH R01MH111629, NIMH R01MH107573, and NIMH R01 MH119430) and the Gustavus and
Louise Pfeiffer Research Foundation.
\newpage
\appendix

\section{Notations}

\begin{table}[ht]
\centering
\caption{Notations used throughout the paper.}
\renewcommand{\arraystretch}{1.3}
\begin{tabular}{cl}
\toprule
\textbf{Notation} & \textbf{Description} \\
\midrule
$\mathbb{R}^n$ & $n$-dimensional Euclidean space \\
$x, \mathbf{x}, \mathbf{X}$ & Scalar, vector, and matrix \\
$x(i)$ & Value of signal $x$ at node $i$ \\
$\mathbf{X}^T$ & Transpose of matrix $\mathbf{X}$ \\
$N$ & Number of nodes in the graph \\
$\mathcal{N}_i$ & Neighborhood of node $i$ \\
$\mathcal{S}$ & Subset of graph nodes $\mathcal{V}$ \\
$\mathbf{A}$ & Adjacency matrix \\
$\mathbf{D}$ & Degree matrix (diagonal) \\
$\mathbf{L}$ & Graph Laplacian matrix \\
$\mathbf{L}_{\mathrm{n}}$ & Normalized Laplacian matrix \\
$\mathbf{U}$ & Matrix of Laplacian eigenvectors \\
$\mathbf{Z}$ & Slepian harmonics matrix \\
$\mathbf{I}$ & Identity matrix \\
$\mathbf{x} * \mathbf{g}$ & Graph convolution of signals $\mathbf{x}$ and $\mathbf{g}$ \\
$\lambda$ & Graph frequency \\
$\mathcal{B}$ & Set of all graph frequencies \\
$\mathcal{W}$ & Selected subset of graph frequencies \\
$\mathbf{S}_{\mathrm{V}}$ & Node selection matrix (diagonal) \\
$\mathbf{S}_{\mathrm{B}}$ & Frequency band selection matrix (diagonal) \\
$\mathbf{\Theta}$ & Diagonal matrix of graph filter responses \\
$\mathbf{h}_i^{(\ell)}$ & Latent feature vector of node $i$ at layer $\ell$ \\
$\mathbf{H}^{(\ell)}$ & Matrix of all latent features at layer $\ell$ \\
$\sigma$ & Non-linear activation function \\
\bottomrule
\end{tabular}
\label{table:notation_gsp}
\end{table}
\section{Related Works}
\label{subsec:related_work}

Traditional Slepians were extended to graph domain by \citet{VilDemPre17}. The graph Slepians were further studied for spectral guided filtering and denoising applications~\cite{PetVan19}, time-resolved analysis of dynamic graphs~\cite{LiMerVan19}, exploratory analysis on fMRI data~\cite{BolFarObe18}, and applications to \textit{C. Elegans} connectomes~\cite{VanDemPre17_Connectome}. 

Existing graph convolutional networks are generally classified into two categories: spectral methods and spatial methods. Most of the spectral methods define graph convolutions via frequency analysis on graphs. Spectral GCN~\cite{BruZarSzl14} is the first work on spectral convolutional networks on graphs, which leveraged graph Fourier transform to learn filters in spectral domain. Later polynomial approximations of spectral filtering were used to define graph convolutions including ChebNet~\cite{DefBreXav16} and first order GCN~\cite{KipWel17}, which allowed to perform convolutions directly in spatial domain. Moreover, graph wavelets have also been used to define graph convolutions~\cite{XuHuaQi18,XuDaiLi22, Tong_LEGS, CHEW2024101635}. Spatial graph convolutional networks define convolutions in spatial domain and have also been studied as message-passing neural networks~\cite{GilSamSch17}. A few works in this direction include graph attention network (GAT)~\cite{VelGuiAra18}, GraphSAGE~\cite{HamYinLes17}, and graph isomorphism network (GIN)~\cite{XuHuLes18}.

There have been many works under the umbrella of subgraph representation learning for subgraph-level tasks. Subgraph neural networks~\cite{AlsFinLi20} learn subgraph embeddings via a routing mechanism that propagates neural messages between its components and randomly sampled anchor patches from the underlying graph. Some other works in this direction include subhypergraph inductive neural network~\cite{Luo22}, generalizing Weisfeiler-Lehman kernels to subgraphs~\cite{KimOh24}, SHINE~\cite{SheYanJu22}, and subgraph representation learning with self-attention and free adversarial training~\cite{QinTanLu24}. \textbf{SlepNet} is the first ever graph convolution network based on graph Slepians, enabling optimal signal energy concentration within relevant laerned subgraphs.
\newtheorem*{prop1}{Proposition}
\section{Proof of Proposition 1}
\label{subsec:proof}
As discussed in Section~\ref{sec:background}, graph Slepians are computed by optimizing energy concentration criterion~\eqref{eq:mu} such that they are bandlimited and have maximum energy over a subset of nodes $\cS \subset \cV$. In contrast, graph Slepians can also be interpreted as signals that are spacelimited to a subset of nodes, while at the same time maximize signal energy in a specific frequency band $\cW \subset \cB$. In this case, the problem is formulated as 
\begin{align}
\label{eq:slep2}
    \underset{\bz}{\mathrm{maximize}} \quad \sum_{\lambda_\ell \in \cW}~[\hat{z}(\lambda_\ell)]^2 & \quad \mathrm{subject~to} \quad z(i) = 0~~ \mathrm{for}~~ i \in \cS.
\end{align}
\begin{prop1}
The graph Slepian harmonics spacelimited to a set of nodes $\cS$ and maximizing energy in a certain frequency band $\cB$ are given by
    \begin{equation}
        \bz = \bS_{\mathrm{V}} \bU \hat{\bz},
    \end{equation}
    where $\hat{\bz}$ are the eigenvectors of the concentration matrix $\bC = \bS_{\mathrm{B}}^T \bU^T \bS_{\mathrm{V}} \bU \bS_{\mathrm{B}}$. Here, $\bS_{\mathrm{B}}$ is the diagonal band selection matrix, $\bS_{\mathrm{V}}$ is the diagonal node selection matrix, and $\bU$ is the Laplacian eigenvector matrix.
\end{prop1}
\begin{proof}
    For spacelimited signal, we can write 
    \begin{align*}
        \sum_{\lambda_\ell \in \cW}~[\hat{z}(\lambda_\ell)]^2 &= \hat{\bz}^T \bS_{\mathrm{B}} \hat{\bz} \\
        &= (\bU ^T \bz)^T \bS_{\mathrm{B}} (\bU ^T \bz) \\
        &= \bz^T \bU \bS_{\mathrm{B}} \bU^T \bz \\
        &= (\bS_{\mathrm{V}} \bz)^T \bU \bS_{\mathrm{B}}  \bU^T \bS_{\mathrm{V}} \bz \\
        &= \bz^T \bS_{\mathrm{V}} \bU \bS_{\mathrm{B}} \bU^T \bS_{\mathrm{V}} \bz.
    \end{align*}
Since the eigenvalues of $ \bS_{\mathrm{V}} \bU \bS_{\mathrm{B}} \bU^T \bS_{\mathrm{V}}$ are same as the eigenvalues of the energy concentration matrix $\bC = \bS_{\mathrm{B}}^T \bU^T \bS_{\mathrm{V}} \bU \bS_{\mathrm{B}}$, therefore, due to spacelimit-ness constraint, the Slepian harmonics solving \eqref{eq:slep2} are given by $\bz = \bS_{\mathrm{V}} \bU \hat{\bz}$.  
\end{proof}
\section{Datasets}
\label{subsec:dataset}

\subsection{Obsessive Compulsive Disorder - Perceptual and Visual Decision Making}
The first dataset used to evaluate our model consists of resting-state fMRI data collected from participants before and after they completed a novel decision-making task known as the Perceptual and Value-based Decision-Making (PVDM) task \cite{ma_ocd_2021}. This task required participants to make judgments using customized visual stimuli in two distinct settings: one focused on perceptual discrimination, and the other on value-based choices. The study sample included both healthy controls (HC) and unmedicated individuals diagnosed with obsessive-compulsive disorder (OCD). Prior analyses of this dataset found that males with OCD, in particular, showed more cautious decision-making and less efficient evidence accumulation compared to matched control participants. To analyze the brain network structure, we used the DiFuMo atlas \cite{Dadi2020} to define a graph where each node corresponds to a cortical or subcortical region, and edges link regions that are spatially adjacent.

\subsection{Obsessive Compulsive Disorder - Risk and Ambiguity}
The second dataset we used is the fMRI data from 51 participants, including 23 unmedicated individuals with obsessive-compulsive disorder (OCD) and 28 healthy controls matched for age, gender, and IQ. Participants completed a decision-making task designed to probe responses to risk and ambiguity (R\&A), based on the paradigm developed by \cite{pushkarskaya2015} and adapted from \cite{levy_neural_2010}. The task consisted of seven blocks, alternating between resting periods, passive viewing of lottery images, and active decision-making. In each trial, participants chose between a guaranteed amount of money and a lottery option with varying levels of uncertainty. This setup allowed us to examine how people make decisions when faced with different kinds of uncertainty.

\subsection{Autism Spectrum Disorder}
Third, we consider the ABIDE (Autism Brain Imaging Data Exchange) dataset \cite{DiMartino2014} that consists of resting-state fMRI data collected from 1,100 participants including both individuals with Autism Spectrum Disorder (ASD) and Healthy Control across 17 different international sites. We consider 40 participants with ASD and 40 healthy controls. In order to construct the brain network, we used the BASC multiscale atlas \cite{BELLEC20101126} to define the graph where each node corresponds to a brain region, and the edges are defined based on functional connectivity.

\subsection{Traffic Dynamics}
Finally, we consider data from the Caltrans Performance Measurement System (PeMS)~\cite{doi:10.3141/1748-12} consisting of real-time traffic measurements across major highways. These datasets capture traffic speed measurements collected every 5 minutes from over 39,000 sensors deployed in the 3rd (PEMS03) and 7th (PEMS07) congressional districts. Each sensor is represented as a node in the graph and the edges are the spatial relationship between the sensors (physical road connections and sensor proximity). The recorded measurements form the temporal graph signals, yielding a time series over the graph. We aim to perform multi-class classification, where each graph is labeled with a day of the week, hence enabling our model to learn patterns in traffic dynamics. 

{\tiny
\begin{table}[ht]
\centering
\scalebox{0.75}{
\begin{tabular}{lcccc}
\toprule
\textbf{Dataset} & \textbf{\#Timepoints} & \textbf{\#Subjects} & \textbf{Task} & \textbf{Data Modality} \\
\midrule
OCD-PVDM~\cite{ma_ocd_2021} 
& 420 
& 51 (28 HC/23 OCD) 
& Perceptual and value-based decision-making 
& Resting-state fMRI \\
OCD-RA~\cite{pushkarskaya2015, levy_neural_2010} 
& 420 
& 51 (28 HC/23 OCD) 
& Risk and ambiguity-based decision-making 
& Resting-state fMRI \\
ASD~\cite{DiMartino2014} 
& Variable 
& 80 (40 HC/40 ASD) 
& Social perception (point-light motion) 
& Resting-based fMRI \\
Traffic-PEMS03~\cite{doi:10.3141/1748-12} 
& 26208 
& N/A 
& Day-of-week traffic classification 
& Traffic sensor network \\
Traffic-PEMS07~\cite{doi:10.3141/1748-12} 
& 28224 
& N/A 
& Day-of-week traffic classification 
& Traffic sensor network \\
\bottomrule
\end{tabular}
}
\caption{Summary of datasets used for model evaluation. Each fMRI subject contributes a graph with node-wise time series; traffic datasets consist of recordings from spatiotemporal sensor networks.}
\label{tab:dataset_summary}
\end{table}
}

\section{Additional Results}
\label{sec:appendix_results}

As described in Section~\ref{sec:main}, \ref{sec:experiments} and seen in Figure~\ref{fig:summary_results}, SlepNet processes graph structured input such as fMRI data by learning relevant subgraphs through an attention-based mask learning module and is followed by Slepian spectral filtering. SlepNet yields three important outputs: (1) a learned mask, which highlights important regions of the brain and provides neuroscientific interpretability. (2) the Tier 1 classification output for neuronal prediction tasks such as distinguishing between OCD patients and healthy controls; and (3) the Slepian-filtered learned representations, which capture important spatiotemporal features and is useful for downstream analysis. We evaluate the quality of each of these outputs. The learned masks are projected onto the brain map to interpret which regions are emphasized by our model. The Tier 1 classification accuracy reflects the model's power in neuronal prediction tasks across all datasets. In order to assess the representational quality of the learned embeddings, we perform further downstream analysis on them. The learned embeddings are visualized using T-PHATE to reveal temporal structure and dynamics and are utilized for more fine-grained Tier 2 classification. Additionally, we perform curvature analysis on the T-PHATE trajectories to quantify changes in neural states over time. Our analysis extends beyond simple classification to offering deeper characterization of Slepian embeddings, hence demonstrating its robustness and expressivity. 

\begin{figure}[ht]
\centering
\includegraphics[width=0.99\linewidth]{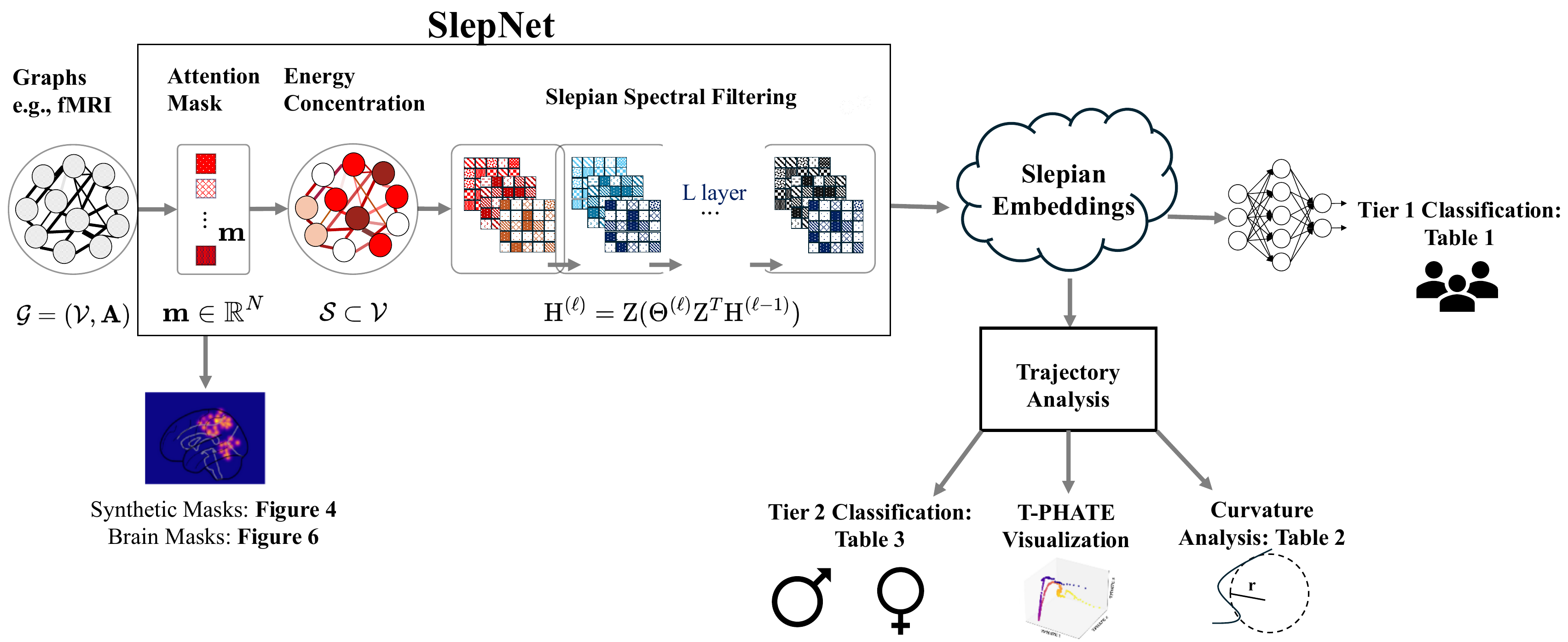}
\caption{Summary of SlepNet Results.}
\label{fig:summary_results}
\end{figure}

\subsection{SlepNet Identifies Relevant Brain Regions through Mask Learning}
\label{subsec:brain_mask_learning}

As described in the Section~\ref{sec:main}, SlepNet architecture features an attention-based mask learning module based on which Slepian spectral filtering is performed. This module learns to assign importance weights to different regions of the brain graph. The learned mask makes our model interpretable by identifying regions of the brain that contribute most to the classification decision. Since the model is trained to classify OCD patients and healthy controls,the learned mask highlights brain regions that are most predictive of the disorder. Figure~\ref{fig:mask} shows the projection of the learned mask onto four different anatomical views of the brain: sagittal left hemisphere, coronal, sagittal right hemisphere, and axial. The learned mask consistently highlights with high importance the parietal lobe, dorsolateral prefrontal cortex, and left temporal across both SlepNet variants and datasets. The regions highlighted align with their roles in differentiating between OCD patients and Healthy Controls as shown in~\cite{MENZIES2008525, 10.3389/fpsyt.2019.00452}. The interpretability of the learned mask is especially important in neuroscience applications, where the ability to highlight important brain regions involved in making accurate predictions. By incorporating the attention-based mask learning module, SlepNet is able to focus on relevant subgraphs of the brain for slepian-based spectral filtering and enabling interpretable insights into neural substrates of the psychiatric condition. 

\begin{figure}[ht]
\centering
\includegraphics[width=0.99\linewidth]{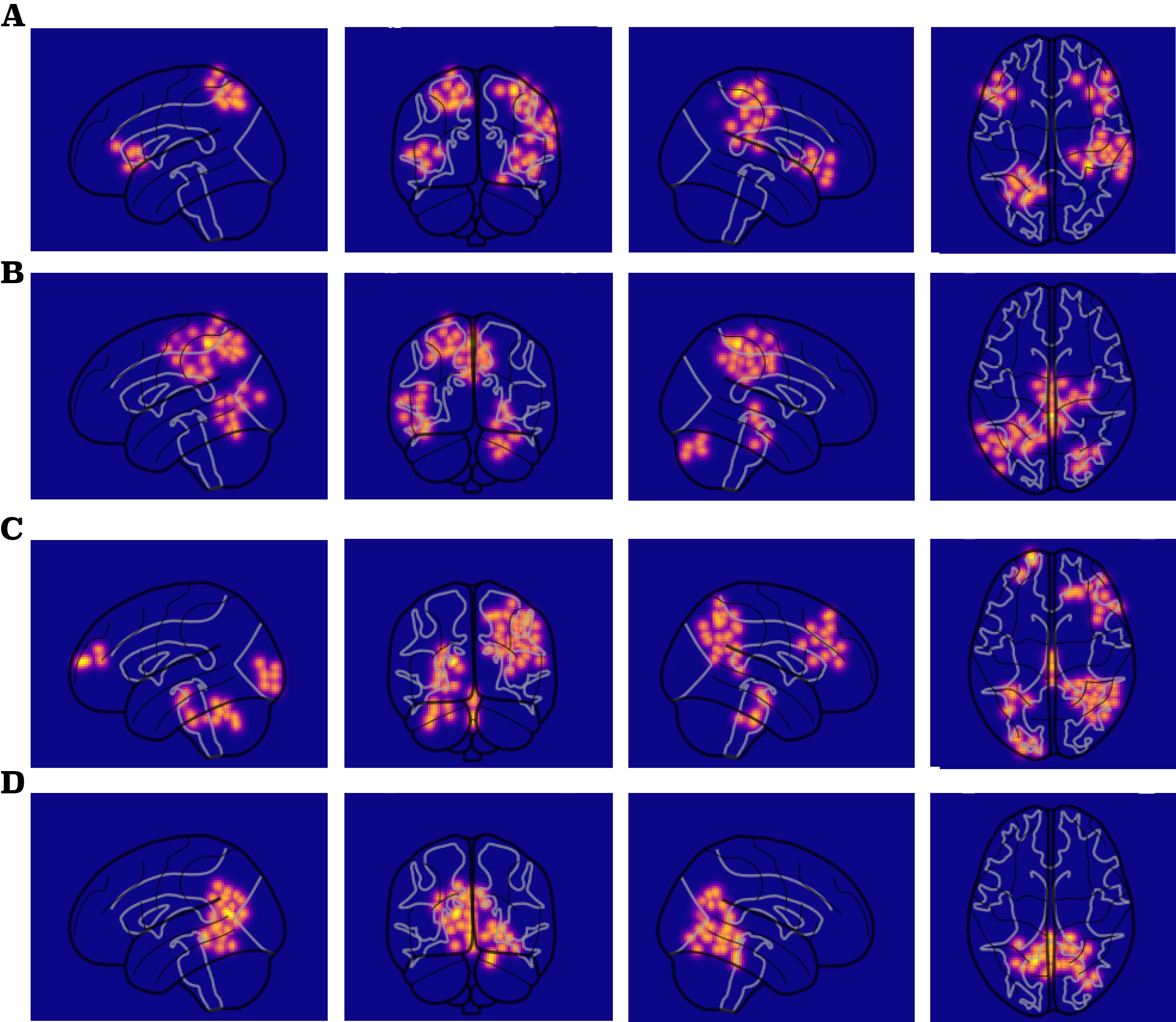}
\caption{Visualization of learned mask projected onto the brain surfaces. (A) SlepNet - I : Energy Concentration on the PVDM dataset, 
(B) SlepNet - I : Energy Concentration on the RA dataset,  (C) SlepNet - II : Embedded Distance on the PVDM dataset, (D) SlepNet - II : Embedded Distance on the RA dataset.
Brighter regions indicate higher learned importance in classification.}
\label{fig:mask}
\end{figure}

\subsection{SlepNet Captures Temporally Coherent Trajectories Across fMRI datasets}
\label{subsec:viz_add}

To further analyze the quality of the representations learned by SlepNet, similar to our visualization analysis in Section~\ref{subsubsec:visualizations}, we present the T-PHATE visualizations of SlepNet, Spectral GCN, and raw fMRI time series embeddings for the OCD-RA and ASD datasets. T-PHATE is a nonlinear dimensionality reduction serves to preserve temporal trajectories on high dimensional data. 
In Figures~\ref{fig:tphate_ra} and \ref{fig:tphate_asd}, each plot represents the trajectory of a subject colored by time. SlepNet embeddings in both datasets exhibit temporally coherent trajectories and are curvier with more variation that reflect the neural dynamics, while Spectral GCN embeddings remain fragmented and disorganized and raw T-PHATE trajectories show limited temporal distinction. Hence the results further show SlepNet's ability to capture structured temporal embeddings across diverse fMRI datasets.

\begin{figure}[ht]
\centering
\includegraphics[width=0.99\linewidth]{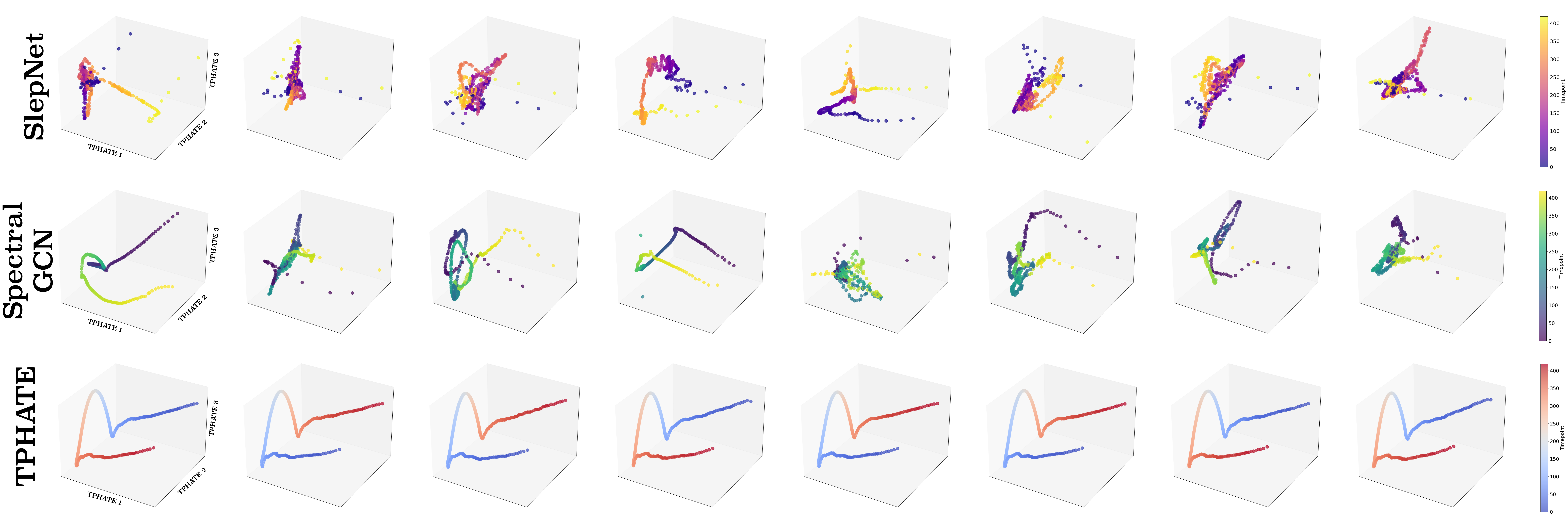}
\caption{T-PHATE visualization (RA dataset) of Graph Slepian embeddings using all nodes colored by timepoint, showing dynamic trajectories.}
\label{fig:tphate_ra}
\end{figure}

\begin{figure}[ht]
\centering
\includegraphics[width=0.99\linewidth]{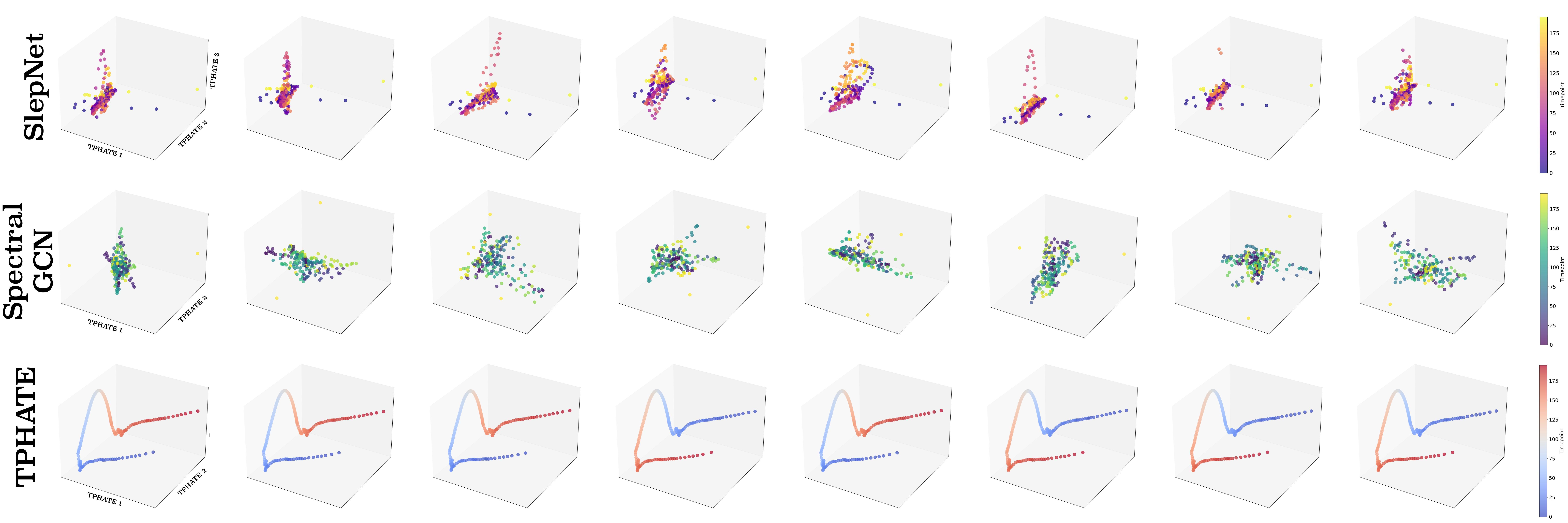}
\caption{T-PHATE visualization (ASD dataset) of Graph Slepian embeddings using all nodes colored by timepoint, showing dynamic trajectories.}
\label{fig:tphate_asd}
\end{figure}

\section{Ablation Studies}
\label{sec:ablations}
In order to investigate the impact of the number of Slepian basis vectors used, we conduct an ablation study across the three fMRI datasets. For each setting, we vary the number of slepian vectors used for spectral filtering ranging from 5 to 500 and report the mean $\pm$ standard deviation of classification accuracy over 10 runs. As seen in Table~\ref{tab:slepnet_ablation}, increasing the number of Slepian vectors consistently improves the performance across all datasets. Notably, SlepNet-I: Energy Concentration using 500 Slepian vectors achieves the best result, outperforming the lower dimensional setting across all three datasets. The performance of the model drastically improves as the number of Slepian vector increases, indicating that expanding the spectral bandwidth allows the model to capture more discriminative frequency components of the signal.  

\begin{table}[ht]
\centering
\caption{Ablation study showing mean $\pm$ standard deviation of classification accuracy (\%) over 10 runs for different numbers of Slepian basis vectors, across three datasets and two SlepNet variants.}
{\renewcommand{\arraystretch}{1.2}
\resizebox{\textwidth}{!}{\begin{tabular}{clccc}
\toprule
\textbf{\# Slepian Vectors} & \textbf{Model} & \textbf{PVDM} & \textbf{RA} & \textbf{ASD} \\
\midrule
\multirow{2}{*}{5}  & SlepNet - I : Energy Concentration & 64.02~$\pm$ 0.85\% & 65.73~$\pm$ 0.85\% & 54.88~$\pm$ 0.65\% \\
                    & SlepNet - II : Embedded Distance   & 61.17~$\pm$ 0.67\% & 62.86~$\pm$ 1.36\% & 52.93~$\pm$ 0.54\% \\
\midrule
\multirow{2}{*}{10} & SlepNet - I : Energy Concentration & 67.96~$\pm$ 0.62\% & 71.06~$\pm$ 0.85\% & 58.58~$\pm$ 0.69\% \\
                    & SlepNet - II : Embedded Distance   & 63.72~$\pm$ 0.86\% & 66.97~$\pm$ 1.73\% & 53.39~$\pm$ 0.49\% \\
\midrule
\multirow{2}{*}{15} & SlepNet - I : Energy Concentration & 70.00~$\pm$ 0.68\% & 74.91~$\pm$ 0.64\% & 60.11~$\pm$ 0.69\% \\
                    & SlepNet - II : Embedded Distance   & 64.34~$\pm$ 1.24\% & 69.32~$\pm$ 1.37\% & 53.52~$\pm$ 0.76\% \\
\midrule
\multirow{2}{*}{20} & SlepNet - I : Energy Concentration & 72.07~$\pm$ 0.31\% & 76.91~$\pm$ 0.56\% & 62.10~$\pm$ 0.58\% \\
                    & SlepNet - II : Embedded Distance   & 65.71~$\pm$ 1.25\% & 71.51~$\pm$ 1.16\% & 53.63~$\pm$ 0.75\% \\
\midrule
\multirow{2}{*}{50} & SlepNet - I : Energy Concentration & 79.73~$\pm$ 0.38\% & 84.73~$\pm$ 0.64\% & 67.31~$\pm$ 1.31\% \\
                    & SlepNet - II : Embedded Distance   & 71.01~$\pm$ 1.35\% & 77.59~$\pm$ 1.15\% & 55.48~$\pm$ 0.88\% \\
\midrule
\multirow{2}{*}{100} & SlepNet - I : Energy Concentration & 84.70~$\pm$ 0.46\% & 90.65~$\pm$ 0.62\% & \textbf{74.13~$\pm$ 0.81\%} \\
                     & SlepNet - II : Embedded Distance   & 74.55~$\pm$ 0.86\% & 81.10~$\pm$ 1.36\% & \textbf{56.89~$\pm$ 0.97\%} \\
\midrule
\multirow{2}{*}{300} & SlepNet - I : Energy Concentration & 91.52~$\pm$ 0.37\% & 95.49~$\pm$ 0.37\% & -~$\pm$ -\% \\
                     & SlepNet - II : Embedded Distance   & 79.82~$\pm$ 0.77\% & 85.76~$\pm$ 0.84\% & -~$\pm$ -\% \\
\midrule
\multirow{2}{*}{400} & SlepNet - I : Energy Concentration & 92.53~$\pm$ 0.40\% & 95.85~$\pm$ 0.33\% & -~$\pm$ -\% \\
                     & SlepNet - II : Embedded Distance   & 80.51~$\pm$ 0.56\% & 86.21~$\pm$ 0.66\% & -~$\pm$ -\% \\
\midrule
\multirow{2}{*}{500} & SlepNet - I : Energy Concentration & \textbf{92.94~$\pm$ 0.46\%} & \textbf{95.97~$\pm$ 0.30\%} & -~$\pm$ -\% \\
                     & SlepNet - II : Embedded Distance   & \textbf{81.62~$\pm$ 0.62\%} & \textbf{86.33~$\pm$ 1.01\%} & -~$\pm$ -\% \\
\bottomrule
\end{tabular}}

\label{tab:slepnet_ablation}}
\end{table}

\section{Experimental Details}
\label{subsec:exp_details}

In the SlepNet architecture, we use $L = 3$ layers ($64$ dimension latent embeddings) as shown in Equation~\ref{eq:slepian}. The nonlinearity $\sigma$ in each layer is set to ReLU. For all datasets, we use $\kappa=60$ as the number of clusters. We performed a grid search to optimize for learning rate and weight decay parameters used in the AdamW optimizer~\cite{loshchilov2019decoupledweightdecayregularization}. We use cross-entropy loss with label smoothing for classification. We trained the model for 300 epochs on every dataset, saved and recorded the results of the model with highest test accuracy during training. All experiments were repeated for 10 iterations to compute mean and standard deviation of the accuracy. The datasets used are split into $80\%$ training and $20\%$ testing sets for each run. The experiments are GPU-accelerated and typically take approximately 1 hour of training on a single NVIDIA A100 GPU using 40GB of memory.




\clearpage
\bibliography{references.bib}

\end{document}